%% file: main.tex
\documentclass[sigconf,natbib=false]{acmart}

\AtBeginDocument{%
  }

\usepackage{algorithm}
\usepackage{algorithmic}
\usepackage{xcolor}
\usepackage{colortbl}
\usepackage{placeins}
\usepackage{multirow}
\usepackage{amsmath}
\usepackage{subcaption}

\definecolor{bestbg}{RGB}{183,239,202}
\newcommand{\best}[1]{\cellcolor{bestbg}\textbf{#1}}
\newcommand{\xhdr}[1]{\vspace{0mm}\noindent{\bfseries #1.}}
\usepackage{xspace}
 
\newcommand{\method}{\textsc{TRicci}\xspace}

\newcommand{\sota}{state-of-the-art\xspace}

\copyrightyear{2026}
\acmYear{2026}
\setcopyright{cc}
\setcctype{by}
\acmConference[CIKM '26]
  {Proceedings of the 35th ACM International Conference on Information and Knowledge Management}
  {November 07--11, 2026}
  {Rome, Italy}
\acmBooktitle{Proceedings of the 35th ACM International Conference on Information and Knowledge Management (CIKM '26), November 07--11, 2026, Rome, Italy}
\acmDOI{10.1145/3799682.3840790}
\acmISBN{979-8-4007-2539-5/2026/11}

\RequirePackage[
  datamodel=acmdatamodel,
  style=acmnumeric,
]{biblatex}

\begin{document}

\title{Edge Sparsification via Temporal Forman-Ricci Curvature for Dynamic Graph Learning}

\author{Poupak Azad}
\orcid{0000-0003-0607-5073}
\affiliation{%
  \institution{University of Manitoba}
  \city{Winnipeg}
  \country{Canada}
}
\email{azadp@myumanitoba.ca}

\author{Cuneyt Gurcan Akcora}
\orcid{0000-0002-2882-6950}
\affiliation{%
  \institution{University of Central Florida}
  \city{Orlando}
  \country{United States}
}
\email{cuneyt.akcora@ucf.edu}

\author{Kiarash Shamsi}
\orcid{0000-0002-2000-3802}
\affiliation{%
  \institution{University of Manitoba}
  \city{Winnipeg}
  \country{Canada}
}
\email{shamsik1@myumanitoba.ca}

\renewcommand{\shortauthors}{Poupak Azad et al.}

\begin{abstract}
Temporal graph learning has become essential for analyzing real-world systems whose interactions continuously evolve over time, including financial transaction networks, communication systems, and online social platforms. However, learning from large-scale temporal graphs remains computationally challenging when networks are dense and rapidly changing. To address this limitation, we propose a network-curvature-inspired edge sparsification framework for dynamic graph learning. Our proposed method, \method, extends classical Forman--Ricci curvature to directed weighted temporal graphs by capturing structural support, temporal recency, and local interaction competition.
  
Experiments on 9 transaction networks and 3 temporal graph benchmark datasets demonstrate that the proposed framework preserves predictive performance across multiple graph-level prediction tasks. The results show that \method sparsifies temporal graphs by approximately 80\% while reducing end-to-end downstream training and inference time by an average of 55.94\%, without substantial degradation in predictive performance. Our findings suggest that temporal curvature can serve as a principled basis for scalable temporal graph learning by preserving predictive temporal-structural information under substantial sparsification.
\end{abstract}

\ccsdesc[500]{Information systems~Data mining}
\ccsdesc[500]{Computing methodologies~Machine learning algorithms}
\ccsdesc[300]{Mathematics of computing~Graph algorithms}
\ccsdesc[300]{Networks~Network structure}

\keywords{Dynamic Graph Learning, Temporal Graphs, Graph Sparsification, Forman--Ricci Curvature}

\maketitle

\section{Introduction}
\label{sec:introduction}

Many real-world systems are represented as temporal graphs, where entities interact through time-stamped edges whose structure evolves continuously~\cite{kazemi2020representation}. Examples include communication systems~\cite{kumar2019predicting}, financial transaction networks~\cite{pareja2020evolvegcn}, online social platforms~\cite{sankar2020dysat}, and recommender systems~\cite{fan2021continuous}. Learning from such graphs is important because temporal interaction patterns often reveal changes in activity, influence, and emerging structural roles. In blockchain transaction networks, for instance, the temporal organization of transfers can expose shifts in user activity, transaction flow, and the emergence or disappearance of influential addresses~\cite{zhao2021temporal,AzadCKAChainlet25}.

Despite their importance, temporal graphs pose computational challenges. As interactions accumulate over time, temporal networks can become large, dense, and rapidly changing~\cite{huang2023temporal}. Downstream graph learning pipelines must repeatedly process evolving edge sets, construct temporal snapshots, extract structural features, and train predictive models over sequences of graph states. This becomes especially expensive when the graph contains many edges that contribute little to the target prediction task. In such settings, using the full graph may increase memory consumption and runtime while also introducing noisy or weakly informative interactions into the learning pipeline.

One way to address this challenge is graph sparsification, where the goal is to retain a compact subset of informative edges while preserving the predictive and structural properties of the original graph~\cite{spielman2008graph}. Existing sparsification and graph condensation methods have shown that smaller graph representations can support efficient learning~\cite{seo2024teddy, zhang2023ricci,zhang2024graph}. However, many existing approaches are designed primarily for static graphs, rely on degree-based or topology-only criteria, or optimize synthetic graph representations without explicitly modeling the temporal role of edges~\cite{gao2025rethinking, wang2024fast}. This is limiting for dynamic transaction networks, where the importance of an interaction depends on its weight, timing, and local temporal interaction context.

Ricci curvature has recently been used as an effective geometric tool~\cite{chen2025graph} for graph learning, as it provides an edge-level measure of local structural organization. Curvature-based graph methods can identify bottleneck edges, redundant connections, community-like regions, and structurally important interactions, making them useful for graph representation learning, robustness analysis, and sparsification. However, existing graph Ricci curvature formulations are mainly defined for static graphs and do not explicitly account for edge timestamps, temporal recency, or time-dependent competition among nearby interactions~\cite{chen2025graph,zhang2023ricci}.

To address this limitation, we propose \method, a temporal Forman--Ricci curvature-driven edge sparsification framework, for efficient dynamic graph learning. Building on the classical Forman--Ricci curvature formulation for discrete geometric structures~\cite{forman2003bochner} and its graph-based extensions~\cite{sreejith2016forman}, our method, \method, assigns each temporal edge an importance score by jointly considering structural support, temporal context, and local interaction competition. \method then retains 
task-relevant curvature-ranked edges selected through validation 
to construct a compact, sparse graph for downstream prediction. We evaluate the framework on multiple temporal graph prediction tasks across blockchain transaction networks and temporal benchmark datasets. The results show that the proposed method preserves strong predictive performance while retaining only approximately 20\% of the original edges and reducing end-to-end runtime by an average of 55.94\%, demonstrating that temporal curvature can serve as an effective edge-selection mechanism for scalable dynamic graph learning.

The main contributions of this paper are summarized as follows:
\begin{itemize}
    \item To the best of our knowledge, we introduce the first temporal Forman--Ricci curvature formulation for directed weighted temporal graphs.
    
    \item We apply curvature-based edge sparsification to discrete-time dynamic graphs by retaining task-relevant curvature-ranked edges to construct compact temporal graph representations. 

\item We evaluate \method on three graph-level temporal prediction tasks across 12 datasets, including blockchain transaction networks and TGBL benchmark graphs, and compare it against three \sota sparsification methods. The results show that high-curvature sparse graphs retain $97.7\pm2.2\%$ of the full-graph ROC--AUC while using only approximately $20\%$ of the original edges, compared to $89.4\pm6.3\%$ for the strongest competing method.
 
\end{itemize}

To support reproducibility, the implementation and experimental artifacts are publicly available at:
\textcolor{blue}{\url{https://github.com/FDataLab/TemporalRicci}}.

\section{Related Work}
\label{sec:related_work}

Temporal graph learning methods are commonly categorized into discrete-time dynamic graphs (DTDGs) and continuous-time dynamic graphs (CTDGs)~\cite{huang2023temporal}. DTDG methods represent temporal evolution using a sequence of graph snapshots constructed over fixed time intervals, while CTDG methods model interactions as continuous event streams with exact timestamps. Snapshot-based DTDG approaches are widely used in temporal prediction and graph evolution analysis because they provide stable structural representations and scalable processing for large dynamic networks~\cite{kazemi2020representation,rossi2020temporal}. In contrast, CTDG methods typically focus on event-driven temporal message passing and continuous interaction modeling~\cite{rossi2020temporal,xu2020inductive}. This distinction is important for sparsification because edge-importance scores may be defined either over snapshot-level graph structures or over individual timestamped events.

Graph sparsification methods~\cite{hashemi2024comprehensive} have been widely studied to reduce graph size, improve computational efficiency, and filter noisy or redundant connectivity. Current graph sparsification methods aim to construct a smaller subgraph that approximates the original graph both structurally and spectrally. In graph neural networks, early sparsification strategies often focused on randomly or structurally removing edges during training. \textsc{DropEdge} randomly removes edges from the input graph at each training epoch, acting both as data augmentation and as a message-passing reducer to alleviate overfitting and oversmoothing~\cite{rong2020dropedge}. \textsc{DropEdge++} extends this idea with structure-aware samplers, including layer-dependent and feature-dependent edge sampling, to better control oversmoothing in deep GCNs~\cite{han2023structure}. While these methods are simple and effective, they mainly operate as training-time regularization mechanisms and do not provide a deterministic edge-importance ranking for constructing a reusable sparse graph.

Several studies learn or refine graph structures directly from data. For instance, Franceschi et al. formulate graph structure learning as a bilevel optimization problem, jointly learning a discrete edge distribution and the parameters of a graph convolutional network~\cite{franceschi2019learning}. This allows the model to infer a sparse dependency structure when the observed topology is noisy, incomplete, or unavailable. \textsc{Pro-GNN} further improves robustness by jointly learning a clean graph and a graph neural network model under sparsity, low-rank, and feature-smoothness constraints~\cite{jin2020graph}. \textsc{PTDNet} introduces a parameterized topological denoising network that removes task-irrelevant edges during training while imposing a low-rank constraint on the sparsified graph~\cite{luo2021learning}. These methods demonstrate that edge selection can improve robustness and predictive performance when the input graph contains noisy or unreliable connections. However, because the sparse structure is learned alongside a particular downstream model, the sparsification process is often task-dependent and computationally more expensive than a preprocessing-based edge-ranking strategy~\cite{zhu2021survey}.

Another group of methods addresses scalability by sampling sparse computation graphs rather than constructing a persistent sparse graph. \textsc{FastGCN} uses importance sampling to reduce the cost of graph convolution~\cite{chen2018fastgcn}. \textsc{LADIES} improves this idea by constructing layer-dependent sampled subgraphs for deep graph convolutional network training~\cite{zou2019layer}. \textsc{GraphSAINT} performs subgraph sampling for scalable inductive learning and normalizes the estimator to support efficient mini-batch training~\cite{zeng2019graphsaint}. These approaches are effective for reducing training cost, especially on large static graphs. Nevertheless, their main goal is efficient mini-batch optimization, and the sampled graphs are not necessarily designed to provide a stable sparsified representation that can be reused across multiple downstream temporal prediction tasks.

Recent work also studies sparsification through graph lottery tickets and adaptive pruning. Chen et al. jointly prune both graph edges and model weights, showing that sparse subgraphs and sparse subnetworks can preserve the performance of dense graph neural networks~\cite{chen2021unified}. \textsc{TEDDY} reduces the cost of graph lottery ticket discovery by using a one-shot degree-based edge sparsification strategy~\cite{seo2024teddy}. Zhang et al. extend this direction by combining multiple sparsifier experts with different pruning criteria and sparsity levels, enabling node-adaptive graph sparsification~\cite{zhang2025graph}. These methods provide strong evidence that a large portion of graph edges can be removed without substantial performance loss. However, they are still mainly developed for static graph settings, where edge importance is estimated from a fixed topology rather than from a sequence of timestamped interactions.

A smaller body of work addresses sparsification in temporal or spatio-temporal settings. Zheng et al. jointly learn a stochastic edge-sampling network and a temporal-structural convolutional classifier, showing that task-driven sparsification can reduce overfitting while preserving discriminative temporal interaction patterns~\cite{zheng2020node}. Duan et al. introduce an adaptive sparsification strategy for spatial-temporal graph neural networks by pruning spatial adjacency matrices in traffic forecasting models~\cite{duan2023localised}. \textsc{STGS} formulates spatio-temporal graph sparsification as a reinforcement learning problem that removes edges while preserving structural and dynamic properties relevant to downstream forecasting~\cite{shabani2025stgs}. \textsc{dyGRASS} studies dynamic spectral graph sparsification from a graph maintenance perspective by maintaining spectral sparsifiers under streaming edge insertions and deletions using localized random walks on GPUs~\cite{yuan2025dygrass}. Although these studies extend sparsification beyond static graphs, they are primarily developed for specific settings such as temporal node classification, spatiotemporal forecasting, spatiotemporal prediction, or dynamic spectral maintenance. They do not directly provide a model-independent edge-ranking mechanism for directed weighted temporal graphs that jointly accounts for structural support and temporal interaction context before downstream learning.

\begin{figure*}
    \centering
    \includegraphics[width=0.98\textwidth]{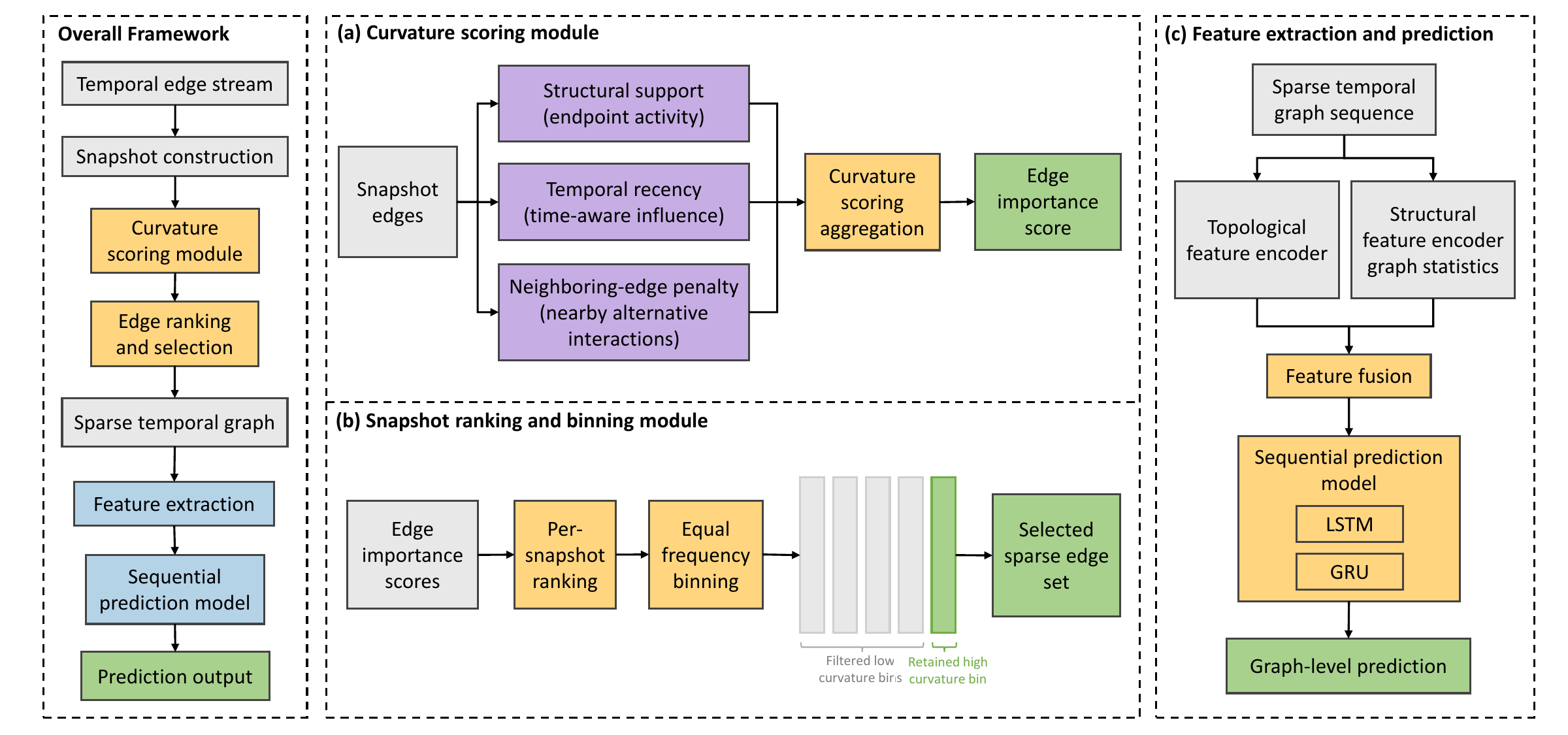}
    \caption{Architecture of the proposed Temporal Forman--Ricci curvature-based sparsification framework in three components: (a) the curvature scoring module, (b) the snapshot ranking and binning module, and (c) the feature extraction and prediction module. The left column summarizes the overall workflow. Gray blocks indicate input graph representations and intermediate data structures, purple blocks denote temporal curvature components, yellow blocks represent scoring, ranking, aggregation, and prediction-processing stages, blue blocks indicate feature extraction and learning modules, and green blocks denote the retained sparse edge set and final prediction output. The retained high-curvature bin is shown as the task-specific setting used in this study; for other downstream tasks, a different curvature range may be selected if it better preserves task-relevant information.}
    \label{fig:tricci_architecture}
\end{figure*}

Nevertheless, sparsifying temporal interaction networks remains a distinct challenge. Existing temporal and spatio-temporal sparsification methods address important settings such as temporal node classification, spatial-temporal forecasting, skeleton-based action recognition, and dynamic spectral maintenance. However, they are not primarily designed to provide an edge-level ranking for directed weighted temporal graphs across different downstream prediction tasks. This leaves a gap for a model-independent temporal sparsification strategy that can score edges before downstream learning while accounting for both structural support and temporal interaction contexts.

\section{Methodology}
\label{sec:methodology}
This section presents the proposed Temporal Forman--Ricci curvature framework for edge sparsification in directed weighted temporal graphs. Our goal is to assign each temporal edge an importance score that reflects its structural support, temporal relevance, and local interaction competition, and then use this score to construct a compact, sparse graph for downstream dynamic graph learning.

\subsection{Problem Setup}
Let $\mathcal{G}=(\mathcal{V},\mathcal{E})$ denote a directed weighted temporal graph, where each edge $e \in \mathcal{E}$ is represented as a tuple

\begin{equation}
e=(u,v,w_e,t_e),
\end{equation}

with source node $u \in \mathcal{V}$, destination node $v \in \mathcal{V}$, positive edge weight $w_e>0$, and timestamp $t_e$. The edge weight may represent the strength of interaction, such as transaction value, review intensity, or interaction frequency.

Given a temporal graph $\mathcal{G}$, our objective is to construct a sparse temporal graph $\widetilde{\mathcal{G}}=(\mathcal{V},\widetilde{\mathcal{E}})$, where $\widetilde{\mathcal{E}}\subset \mathcal{E}$ and $|\widetilde{\mathcal{E}}|\ll |\mathcal{E}|$, such that $\widetilde{\mathcal{G}}$ preserves the most informative temporal interactions for downstream prediction. We sparsify at the edge level: edges are ranked by their Temporal Forman--Ricci curvature scores, and the retained curvature range is selected through validation for the downstream task.

Figure~\ref{fig:tricci_architecture} illustrates the proposed framework. We organize the framework into three components: a curvature-scoring module, a ranking-and-selection module, and a downstream feature-extraction-and-prediction module. We first partition the temporal edge stream into discrete-time snapshots, compute curvature-based edge importance scores within each snapshot, select task-relevant curvature-ranked edges to construct sparse temporal graphs, and then apply the same graph-level feature extraction and sequential prediction pipeline used in the downstream evaluation.

\subsection{Temporal Forman--Ricci Curvature}
Classical Forman--Ricci curvature provides an edge-centered geometric measure of local graph structure. For a static undirected weighted graph, the graph-based Forman--Ricci curvature of an edge \(e=(u,v)\) can be written~\cite{forman2003bochner,sreejith2016forman} as:

\begin{equation}
F(e)=w_e\left(
\frac{w_u}{w_e}
+
\frac{w_v}{w_e}
-
\sum_{e_u \sim e}
\frac{w_u}{\sqrt{w_e w_{e_u}}}
-
\sum_{e_v \sim e}
\frac{w_v}{\sqrt{w_e w_{e_v}}}
\right),
\end{equation}

where \(w_e\) is the edge weight, \(w_u\) and \(w_v\) are the weights of the endpoint nodes, and \(e_u\) and \(e_v\) denote neighboring edges incident to \(u\) and \(v\), respectively. This formulation follows a support-minus-penalty structure: the endpoint terms increase the curvature of \(e\), while neighboring edges around the endpoints reduce it.

The classical graph formulation is static and does not directly model edge direction, timestamps, or time-dependent local interaction patterns. We therefore introduce a temporal directed adaptation. Our Temporal Forman--Ricci curvature score keeps the support-minus-penalty principle, but modifies the two components as follows. \textbf{First}, we replace static node weights with log-smoothed endpoint activity computed inside each temporal snapshot. \textbf{Second}, we define directed outgoing neighborhoods around the endpoints of each edge. \textbf{Third}, and most importantly, we weight neighboring-edge penalties by temporal proximity, since temporally adjacent interactions often provide overlapping local structural information and should therefore contribute more strongly to the competition penalty. This follows the Forman-Ricci support-minus-penalty principle, where neighboring interactions around an edge reduce its curvature when they represent local competition.

We first define the endpoint-support component of the proposed temporal curvature score. For a node \(x \in \mathcal{V}\), let \(\mathcal{E}(x)\) denote the set of edges incident to \(x\) within the current temporal snapshot. We define the weighted activity of \(x\) as

\begin{equation}
\operatorname{strength}(x)=\sum_{e' \in \mathcal{E}(x)} w_{e'}.
\end{equation}

Since temporal interaction networks often contain heavy-tailed edge weights and highly active nodes, we use a log-smoothed node activity term,

\begin{equation}
s(x)=\log\left(1+\operatorname{strength}(x)\right).
\end
{equation}

This transformation reduces the dominance of highly active nodes while preserving relative activity differences. Using this smoothed activity, we define the endpoint support of edge \(e=(u,v,w_e,t_e)\) as

\begin{equation}
S_e = w_e\left(\frac{1}{s(u)}+\frac{1}{s(v)}\right).
\end{equation}

The support term assigns larger values to edges whose endpoints are not dominated by many high-weight incident interactions. In this sense, \(S_e\) measures how distinctive the edge is relative to the local activity around its endpoints.

We next define the neighboring-edge penalty. To make this penalty time-aware, we use a temporal decay kernel between edge \(e\) and a neighboring edge \(e'\):

\begin{equation}
K(e,e',\tau)=\exp\left(-\frac{|t_e-t_{e'}|}{\tau}\right),
\end{equation}

where $\tau>0$ is a temporal scale parameter. When two edges occur close in time, $K(e,e',\tau)$ is close to one; when they are far apart, the influence of $e'$ on $e$ decays toward zero. Thus, the kernel allows nearby interactions to contribute more strongly to the neighboring-edge penalty than temporally distant interactions.

Because the graph is directed, we construct the neighboring-edge penalty from outgoing temporal neighborhoods around the two endpoints of the target edge. For edge $e=(u,v,w_e,t_e)$, we define the outgoing neighborhood of the source endpoint \(u\) as
\begin{equation}
\mathcal{N}_u(e)=\{e'=(u,x,w_{e'},t_{e'}) \in \mathcal{E} \mid x \neq v\},
\end{equation}

and the outgoing neighborhood of the destination endpoint \(v\) as

\begin{equation}
\mathcal{N}_v(e)=\{e'=(v,x,w_{e'},t_{e'}) \in \mathcal{E} \mid x \neq u\}.
\end{equation}

Edges in \(\mathcal{N}_u(e)\) represent source-side alternative interactions, while edges in \(\mathcal{N}_v(e)\) capture destination-side continuation interactions. Together, these neighborhoods describe the local temporal context of \(e\). Temporally close neighboring interactions increase the penalty, indicating that the target edge is less isolated in its local temporal context.

We compute the temporal neighboring-edge penalty from endpoint \(u\) as

\begin{equation}
D_u(e)=
\begin{cases}
\displaystyle
\frac{w_e}{|\mathcal{N}_u(e)|}
\sum_{e'\in \mathcal{N}_u(e)}
\frac{K(e,e',\tau)}{\sqrt{w_e w_{e'}}},
& |\mathcal{N}_u(e)|>0,\\[10pt]
0, & |\mathcal{N}_u(e)|=0.
\end{cases}
\end{equation}

We compute the corresponding penalty from endpoint \(v\) as

\begin{equation}
D_v(e)=
\begin{cases}
\displaystyle
\frac{w_e}{|\mathcal{N}_v(e)|}
\sum_{e'\in \mathcal{N}_v(e)}
\frac{K(e,e',\tau)}{\sqrt{w_e w_{e'}}},
& |\mathcal{N}_v(e)|>0,\\[10pt]
0, & |\mathcal{N}_v(e)|=0.
\end{cases}
\end{equation}

The factor $\sqrt{w_e/w_{e'}}$ preserves the relative weight normalization used in weighted Forman--Ricci curvature; thus, $w_{e'}$ acts as a normalization term, while the temporal kernel controls the proximity-based contribution of the neighboring edge.

The total temporal competition penalty is then

\begin{equation}
D_e=\frac{D_u(e)+D_v(e)}{2}.
\end{equation}

Finally, we define the Temporal Forman--Ricci curvature score of edge \(e\) as
$\operatorname{\method}(e)=S_e-D_e$. Equivalently,

\begin{equation}
\operatorname{\method}(e)
=
w_e\left(\frac{1}{s(u)}+\frac{1}{s(v)}\right)
-
\frac{1}{2}
\left[
D_u(e)+D_v(e)
\right].
\end{equation}

\method is a real-valued curvature score, with
$\method(e) \in \mathbb{R}$. A high \method value indicates that an edge has strong endpoint support and a limited temporally proximate neighboring-edge penalty. A low \method value indicates that the edge has weak endpoint support or appears in a crowded temporal neighborhood where several nearby interactions provide similar local information.

\method is intended as a temporal edge-importance functional inspired by the support-minus-competition structure of Forman--Ricci curvature. We do not claim that the proposed score preserves all geometric interpretations associated with classical graph curvature. Instead, the formulation adapts the Forman-style support-minus-penalty principle to rank directed weighted temporal interactions for downstream sparsification and prediction.

Figure~\ref{fig:toy_low_curvature} illustrates how the proposed curvature score identifies a low-curvature edge within a single temporal snapshot (We report curvature scores of all edges in our GitHub repo). The target edge $e_9$ connects nodes $u$ and $v$, but it is surrounded by several local alternatives around the same endpoints. In particular, outgoing edges from $u$ and $v$, such as $e_7$, $e_8$, $e_{12}$, and $e_{13}$, act as direct competing interactions in the local neighborhood. When these competing edges are temporally close to the target edge, they increase the temporal competition penalty. In addition, incident activity around $u$ and $v$ contributes to endpoint strength, which can reduce the support term of the target edge. As a result, $e_9$ is selected for removal during sparsification.

\begin{figure}[t]
    \centering
    \includegraphics[width=0.96\linewidth]{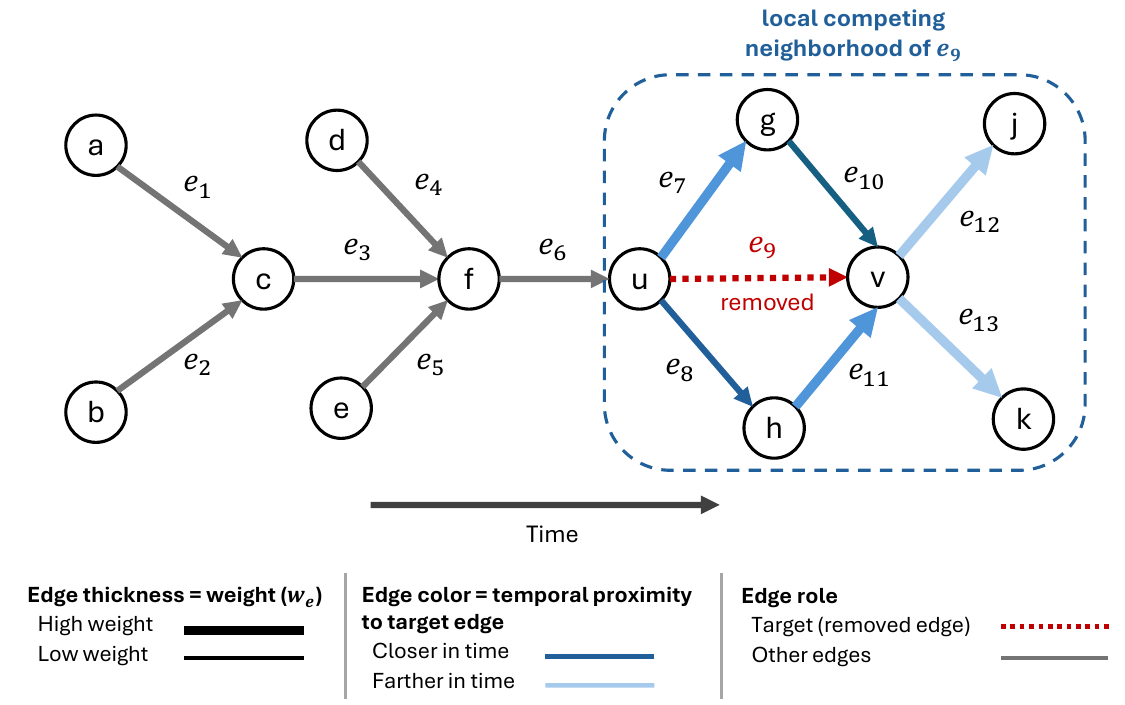}
    \caption{{Single temporal snapshot.} Illustration of low-curvature edge removal within a single temporal snapshot. Edge thickness indicates interaction weight, while edge color indicates temporal proximity to the target edge. The red dashed edge is the target edge selected for removal because it has weak curvature relative to its local competing neighborhood.} 
    \label{fig:toy_low_curvature}
    \vspace{-12pt}
\end{figure}

\subsection{Snapshot-Level Curvature Computation}

Temporal graphs evolve continuously, and computing curvature over the entire history may mix interactions from different temporal contexts. We avoid this issue by representing the temporal edge stream as a sequence of non-overlapping snapshots. Let $\Delta$ denote the snapshot length. We define the edge set of snapshot $r$ as

\begin{equation}
\mathcal{E}_r=\{e \in \mathcal{E} \mid t_e \in [T_r,T_r+\Delta)\}.
\end{equation}

Each snapshot defines a local temporal graph $\mathcal{G}_r=(\mathcal{V}_r,\mathcal{E}_r)$, where $\mathcal{V}_r$ contains the nodes incident to edges in $\mathcal{E}_r$.

We compute Temporal Forman--Ricci curvature independently within each snapshot. This prevents information leakage across future snapshots and ensures that each edge is scored relative to its local temporal context. 

\subsection{Task-Guided Curvature-Based Edge Selection}

We use \method as a model-independent scoring mechanism for temporal edge selection, while the retention policy remains task-guided. Algorithm~\ref{alg:tricci} summarizes the proposed task-guided sparsification procedure.  For each temporal snapshot, \method assigns a curvature score to every edge by combining endpoint support, temporal proximity, and local outgoing interaction competition. The resulting scores define an ordering over temporal interactions within the snapshot.

The sparsification policy is separated from the curvature definition. After computing the snapshot-level scores, the downstream task determines which curvature range should be retained. A policy may retain high-curvature edges, low-curvature edges, or an intermediate curvature band, depending on the validation criterion for task $\mathcal{T}$. At a fixed $\tau$, curvature scores are computed once and reused across the candidate ranges, so the additional selection cost is limited to downstream validation. Based on the aggregate validation results, we use high-curvature retention as the common default, although the optimal range may vary across tasks.

Let $\rho \in [0,1)$ denote the edge-deletion ratio. For snapshot $\mathcal{E}_r$, let $\mathbf{q}_r=\{q_e:e\in\mathcal{E}_r\}$ denote the \method score vector, and let $k_r=\lceil(1-\rho)|\mathcal{E}_r|\rceil$ denote the number of retained edges. We define a task-guided curvature-selection policy $\Pi_{\rho}$ that retains $k_r$ edges according to their curvature ranks for task $\mathcal{T}$. The sparse snapshot is then given by $\widetilde{\mathcal{E}}_r=\Pi_{\rho}(\mathcal{E}_r,\mathbf{q}_r,\mathcal{T})$.

Several policies can be obtained from this definition. A high-curvature policy retains the top-ranked edges, $\Pi_{\rho}^{\mathrm{high}}(\mathcal{E}_r,\mathbf{q}_r,\mathcal{T})=\operatorname{Top}_{k_r}(\mathcal{E}_r;\mathbf{q}_r)$, where $\operatorname{Top}_{k_r}$ returns the $k_r$ edges with the largest \method scores. A low-curvature policy retains the bottom-ranked edges, $\Pi_{\rho}^{\mathrm{low}}(\mathcal{E}_r,\mathbf{q}_r,\mathcal{T})=\operatorname{Bottom}_{k_r}(\mathcal{E}_r;\mathbf{q}_r)$, where $\operatorname{Bottom}_{k_r}$ returns the $k_r$ edges with the smallest \method scores. Instead of high- or low-selection, an alternative approach could use a band-selection policy that first partitions edges into curvature-ranked bins and then retains the bin or set of bins selected for task $\mathcal{T}$. 

Since edge selection is performed independently within each snapshot, the retained sparse graph preserves temporal coverage across the full observation period. This avoids selecting edges only from dense or high-activity snapshots and ensures that each time interval contributes to the final sparse temporal representation.

\begin{algorithm}
\caption{\method: Temporal Forman--Ricci Sparsification}
\label{alg:tricci}
\begin{algorithmic}[1]
\REQUIRE Directed weighted temporal graph $\mathcal{G}=(\mathcal{V},\mathcal{E})$, snapshot length $\Delta$, temporal scale $\tau$, edge-deletion ratio $\rho$, downstream task $\mathcal{T}$, curvature-selection policy $\Pi_{\rho}$
\ENSURE Sparse temporal graph $\widetilde{\mathcal{G}}=(\mathcal{V},\widetilde{\mathcal{E}})$
\STATE Split $\mathcal{E}$ into non-overlapping ${\mathcal{E}_r}$ using $\Delta$
\STATE Initialize $\widetilde{\mathcal{E}} \leftarrow \emptyset$
\FOR{each temporal snapshot $\mathcal{E}_r$}
\STATE Construct local graph $\mathcal{G}_r=(\mathcal{V}_r,\mathcal{E}_r)$
\STATE Compute log-smoothed node strengths $s(x)$ for all $x \in \mathcal{V}_r$
\FOR{each edge $e=(u,v,w_e,t_e) \in \mathcal{E}_r$}
\STATE Compute endpoint support $S_e$
\STATE Identify outgoing competing neighborhoods $\mathcal{N}_u(e)$ and $\mathcal{N}_v(e)$
\STATE Compute temporal penalties $D_u(e)$ and $D_v(e)$ using $K(e,e',\tau)$
\STATE Compute the \method score $q_e=S_e-\frac{D_u(e)+D_v(e)}{2}$
\ENDFOR
\STATE Let $\mathbf{q}_r={q_e:e\in\mathcal{E}_r}$
\STATE Select task-relevant curvature-ranked edges $\widetilde{\mathcal{E}}_r=\Pi_{\rho}(\mathcal{E}_r,\mathbf{q}_r,\mathcal{T})$
\STATE Retain $\widetilde{\mathcal{E}}_r$ as the sparse snapshot for task $\mathcal{T}$
\STATE $\widetilde{\mathcal{E}} \leftarrow \widetilde{\mathcal{E}} \cup \widetilde{\mathcal{E}}_r$
\ENDFOR
\RETURN $\widetilde{\mathcal{G}}=(\mathcal{V},\widetilde{\mathcal{E}})$
\end{algorithmic}
\end{algorithm}

\xhdr{Analytical properties} The proposed curvature score has two elementary monotonicity properties that formalize its intended interpretation. The first property shows that the support term decreases when either endpoint becomes more active. The second property shows that the temporal penalty is strongest for nearby interactions and weakens as neighboring edges move farther away in time. Together, these properties justify the use of TRicci as a support-minus-competition score for directed weighted temporal graphs.

\begin{proposition}[endpoint activity monotonicity] Holding $w_e$ and all penalty terms fixed, the support term decreases strictly when the endpoint activity at $u$ or $v$ increases.
\end{proposition}

\begin{proof} Let $a_u=\operatorname{strength}(u)$ and $a_v=\operatorname{strength}(v)$, with $a_u,a_v>0$. Since $s(x)=\log(1+\operatorname{strength}(x))$, the support term can be written as $S_e=w_e(\log(1+a_u)^{-1}+\log(1+a_v)^{-1})$. Differentiating with respect to $a_u$ gives $\partial S_e/\partial a_u=-w_e((1+a_u)(\log(1+a_u))^2)^{-1}<0$. The same calculation gives $\partial S_e/\partial a_v=-w_e((1+a_v)(\log(1+a_v))^2)^{-1}<0$. Hence any increase in endpoint activity lowers the support assigned to $e$. This proves that edges become less distinctive when their endpoints are crowded by other weighted interactions.
\end{proof}

\begin{proposition}[temporal decay and limiting regimes] Fix the snapshot, edge weights, and directed neighborhood sets. Then each individual penalty contribution decreases strictly with temporal distance. Moreover, as $\tau\to 0^+$, only exactly synchronous neighbors contribute to the penalty. If no timestamp ties occur inside the snapshot, then $\operatorname{TRicci}(e)\to S_e$. As $\tau\to\infty$, the score converges to a time-agnostic directed support-minus-average-competition score.
\end{proposition}

\begin{proof} Consider a nonempty source-side neighborhood $N_u(e)$ and a neighboring edge $e'\in N_u(e)$. Let $d=|t_e-t_{e'}|$. The penalty contribution of $e'$ to $D_u(e)$ is $|N_u(e)|^{-1}\sqrt{w_e/w_{e'}}\exp(-d/\tau)$. Its derivative with respect to $d$ is $-|N_u(e)|^{-1}\tau^{-1}\sqrt{w_e/w_{e'}}\exp(-d/\tau)<0$, since $w_e>0$, $w_{e'}>0$, and $\tau>0$. The same argument applies to every destination-side contribution in $N_v(e)$. Thus, each individual neighboring-edge penalty decreases strictly as temporal distance increases. For the first limit, $\exp(-d/\tau)\to \mathbf{1}_{d=0}$ as $\tau\to 0^+$. Therefore only neighbors with $t_e=t_{e'}$ contribute. If no timestamp ties occur inside the snapshot, then every neighboring edge has $d>0$, so $D_u(e)\to 0$ and $D_v(e)\to 0$, which gives $\operatorname{TRicci}(e)\to S_e$. For the second limit, $\exp(-d/\tau)\to 1$ as $\tau\to\infty$ for every fixed $d$. Hence, temporal distances no longer affect the penalty, and $\operatorname{TRicci}(e)$ converges to the directed weighted score obtained by replacing every temporal kernel factor with one. As $\tau\to\infty$, $\mathrm{TRicci}(e)\to S_e-\frac{1}{2}(|N_u(e)|^{-1}\sum_{e'\in N_u(e)}\sqrt{w_e/w_{e'}} +
|N_v(e)|^{-1}\sum_{e'\in N_v(e)}\sqrt{w_e/w_{e'}}).$
\end{proof}

\xhdr{Computational complexity}
Let $m_r=|\mathcal{E}_r|$ denote the number of edges in snapshot $r$. For each snapshot, computing node activities and constructing outgoing adjacency lists requires $O(m_r)$ time. Ranking edges by their curvature scores requires $O(m_r \log m_r)$ time. The main additional cost comes from computing the neighboring-edge penalty, where each edge is compared with outgoing neighboring edges around its endpoints. The per-snapshot complexity is therefore $O(m_r \log m_r + C_r)$, where $C_r$ is the total cost of scanning local outgoing neighborhoods in snapshot $r$. Since curvature is computed independently across snapshots, the total runtime is the sum of this cost over all snapshots. After the scores are computed, the selection step is lightweight because it applies the policy $\Pi_{\rho}$ to the ranked edge list.
\section{Evaluation}
\label{sec:evaluation}

We evaluate \method through experiments designed to assess predictive preservation, curvature-selection behavior, runtime reduction, and sensitivity to temporal parameters. We first describe the datasets, tasks, baselines, and evaluation protocol. We then analyze curvature-ranked edge bins to determine which curvature ranges preserve the strongest downstream signal. Based on this analysis, we select the default sparsification policy used in the baseline comparison and runtime experiments.

\subsection{Experimental Setup}

\xhdr{Prediction tasks} We evaluate sparsification quality using three graph-level temporal prediction tasks. These tasks are designed to measure whether a sparse graph preserves information about future changes in network activity, influential nodes, and participation breadth.

\paragraph{Task 1: Network Activity Growth Prediction.}
This task predicts whether the number of edges increases in the next temporal snapshot. Let $\mathcal{E}_t$ and $\mathcal{E}_{t+1}$ denote the edge sets in the current and next snapshots, respectively. The binary label is defined as

\begin{equation}
y_t =
\begin{cases}
1, & \text{if } |\mathcal{E}_{t+1}| > |\mathcal{E}_t|,\\
0, & \text{otherwise.}
\end{cases}
\end{equation}

This task evaluates whether the sparse graph preserves signals associated with future expansion in network activity.

\paragraph{Task 2: Network Participation Expansion Prediction.}
This task predicts whether the number of active nodes increases in the next temporal snapshot. Let $\mathcal{V}^{act}_t$ and $\mathcal{V}^{act}_{t+1}$ denote the sets of nodes that participate in at least one edge in the current and next snapshots, respectively. The binary label is defined as

\begin{equation}
y_t =
\begin{cases}
1, & \text{if } |\mathcal{V}^{act}_{t+1}| > |\mathcal{V}^{act}_{t}|,\\
0, & \text{otherwise.}
\end{cases}
\end{equation}

This task evaluates whether the sparsified graph preserves information about future changes in the breadth of network participation.

\paragraph{Task 3: Influential Node Turnover Prediction.}
This task predicts whether the set of influential nodes changes substantially in the next temporal snapshot. For each snapshot, nodes are ranked by degree, and the top-$k$ nodes are treated as influential nodes. Let $\mathcal{T}_t$ and $\mathcal{T}_{t+1}$ denote the top-$k$ node sets in the current and next snapshots, respectively. The turnover ratio is defined as

\begin{equation}
r_t = \frac{|\mathcal{T}_{t+1} \setminus \mathcal{T}_t|}{k}.
\end{equation}

The binary label is

\begin{equation}
y_t =
\begin{cases}
1, & \text{if } r_t > \theta,\\
0, & \text{otherwise.}
\end{cases}
\end{equation}

\medskip
\xhdr{Datasets} We evaluate \method on 12 directed weighted temporal graphs: nine blockchain token-transfer networks from GraphPulse~\cite{shamsi2024graphpulse} and three TGBL temporal graph benchmarks~\cite{huang2023temporal}. Table~\ref{tab:dataset_stats} summarizes the dataset statistics. In the blockchain networks, nodes are wallet addresses and directed temporal edges are token transfers. The TGBL datasets are \textit{TGBL-COIN}, \textit{TGBL-REVIEW}, and \textit{TGBL-COMMENT}, covering cryptocurrency transfers, product reviews, and comment interactions. These datasets test whether \method preserves graph-level predictive signals across temporal graphs with different sizes, densities, durations, and interaction semantics. 

\begin{table}[t]
\centering
\scriptsize
\caption{Statistics of the temporal graph datasets.}
\label{tab:dataset_stats}
\setlength{\tabcolsep}{3pt}
\renewcommand{\arraystretch}{0.95}
\resizebox{\columnwidth}{!}{
\begin{tabular}{lrrrrr}
\toprule
Dataset & Nodes & Edges & Start & End & Days \\
\midrule
ADX & 14,567 & 123,755 & 2020-08-03 & 2023-11-04 & 1,189 \\
BAG & 11,860 & 122,634 & 2023-01-10 & 2023-11-04 & 299 \\
BEPRO & 26,521 & 120,261 & 2020-09-28 & 2023-11-04 & 1,133 \\
DERC & 24,277 & 111,205 & 2021-08-02 & 2023-11-04 & 825 \\
DINO & 15,837 & 94,140 & 2022-11-11 & 2023-11-04 & 359 \\
ETH2x-FLI & 11,008 & 199,088 & 2021-03-14 & 2023-11-04 & 966 \\
EVERMOON & 7,552 & 79,868 & 2023-05-25 & 2023-11-04 & 164 \\
GLM & 53,385 & 234,912 & 2020-11-19 & 2023-11-04 & 1,081 \\
HOICHI & 5,075 & 77,361 & 2022-08-25 & 2023-11-04 & 437 \\
TGBL-REVIEW & 352,637 & 4,873,540 & 1999-06-13 & 2018-10-04 & 7,054 \\
TGBL-COIN & 638,486 & 22,809,486 & 2022-04-01 & 2022-11-01 & 214 \\
TGBL-COMMENT & 994,790 & 44,314,507 & 2005-12-12 & 2010-12-31 & 1,845 \\
\bottomrule
\end{tabular}
}
\end{table}

\xhdr{Baselines}
We compare \method with three recent and prominent graph sparsification baselines: MoG~\cite{zhang2025graph}, TEDDY~\cite{seo2024teddy}, and SEM~\cite{zhang2023ricci}. MoG uses a mixture-of-experts strategy to select pruning criteria and sparsity levels for local graph regions. TEDDY performs one-shot edge sparsification through degree-based structural discrimination. SEM uses Ricci-curvature-based sparsification to preserve informative topological structure for continual graph representation learning. These methods provide strong topology-based comparison points, but their edge-scoring rules do not explicitly use timestamps or temporal interaction context. For a controlled comparison, we apply all baselines under the same snapshot-level protocol and edge-retention budget. \method differs by defining edge importance directly on directed weighted temporal graphs through endpoint support, temporal proximity, and local outgoing interaction competition. Extending the baselines with native timestamp-dependent objectives or temporal neighborhood definitions would require separate temporal sparsification models, which is beyond the scope of this work.

\xhdr{Evaluation setting} The downstream learning pipeline follows the temporal graph property prediction setting introduced by GraphPulse~\cite{shamsi2024graphpulse}, since this setting directly matches our objective of evaluating whether sparse temporal graphs preserve graph-level predictive signals over time. For each sparsified 7-day weekly snapshot, graph-level task features are extracted from the corresponding graph representation and passed to a sequential (i.e., LSTM and GRU-based) prediction framework~\cite{shamsi2024graphpulse}.  To preserve temporal ordering and avoid information leakage, the data were divided using a chronological split, with 70\% used for training, 15\% for validation, and 15\% for testing. The validation split was used to select the retained curvature range and sparsification level, while the test split was reserved for final evaluation. In addition, curvature scores and edge selection were computed independently within each snapshot, so interactions from later snapshots were not used when scoring or selecting edges in an earlier snapshot. All methods are evaluated under the same edge-retention setting, where approximately $\rho = 80\%$ of the original edges are deleted per snapshot. We set $k$ to the top 10\% of active nodes and use $\theta=0.30$ to identify meaningful changes in influential-node composition.
The same feature extraction and prediction pipeline is applied to the full graph, the proposed sparse graph, and all baseline sparsification methods to ensure a fair comparison. Our code is available at \url{https://github.com/FDataLab/TemporalRicci}.

\xhdr{Evaluation metrics}
We evaluate predictive performance using ROC--AUC for each graph-level temporal prediction task. To compare sparse and full graphs across datasets, we also report the full-graph-normalized ROC--AUC, defined as $R_{\mathrm{AUC}}=\mathrm{AUC}(\widetilde{\mathcal{G}})/\mathrm{AUC}(\mathcal{G})$. Values close to $1$ indicate that the sparse graph preserves the predictive performance of the full graph. Sparsification is measured using edge removal and edge retention percentages, defined as $100(1-|\widetilde{\mathcal{E}}|/|\mathcal{E}|)$ and $100|\widetilde{\mathcal{E}}|/|\mathcal{E}|$, respectively. The edge-removal ratio $\rho \in [0,1)$ controls the fraction of deleted edges, so each sparse graph removes approximately $100\rho\%$ of edges and retains approximately $100(1-\rho)\%$. Runtime efficiency is measured as end-to-end reduction relative to the full-graph pipeline, $100(1-T_{\mathrm{sparse}}/T_{\mathrm{full}})$, where $T$ is end-to-end and includes edge selection, feature extraction, and downstream prediction.

\xhdr{Hyperparams} The default temporal decay parameter is set to $\tau=1$ day, which provides the strongest overall performance in the sensitivity analysis. For datasets with different temporal characteristics, $\tau$ can be selected using the validation split. We train the prediction model using the Adam optimizer with a learning rate of \(10^{-4}\) and binary cross-entropy loss. The output layer uses a sigmoid activation for binary prediction. All experiments were run for 100 epochs without learning-rate decay or early stopping.

\xhdr{Computational Environment} Runtime measurements were conducted under a fixed single-machine environment to ensure a controlled comparison across methods. The experiments were executed on an Intel(R) Core(TM) Ultra 9 285H CPU with 16 physical cores, 31.4 GB RAM, and an NVIDIA RTX PRO 1000 GPU with 8 GB memory.

\subsection{Experimental Results}
\label{sec:baseline_comparison}

We report the blockchain transaction network results in Figure~\ref{fig:token_tasks} and the TGBL benchmark results in Figure~\ref{fig:tgbl_tasks}. \method uses the {$\Pi_{\rho}^{\mathrm{high}}$ curvature-selection strategy} {(the rationale behind this strategy, along with the broader analysis of curvature-ranked deletion policies, is examined shortly in Section~\ref{sec:curvature_selection})}.

\input{images/threetaskimages}

\input{images/tgblimages}

Figure~\ref{fig:token_tasks} and Figure~\ref{fig:tgbl_tasks} show that the proposed method consistently preserves strong predictive performance across the evaluated temporal graph prediction tasks under substantial edge sparsification. Across all reported dataset-task combinations, \method achieves the best ROC--AUC result in 27 out of 30 cases. On the TGBL benchmarks, \method achieves the strongest performance in all 6 evaluated cases. The improvements are especially consistent on the blockchain transaction networks, where interaction patterns evolve rapidly and many temporal edges are structurally or temporally redundant.

\begin{table}[t]
\centering
\small
\caption{Average ROC--AUC preservation ($R_{\mathrm{AUC}}$) relative to the full graph on all datasets and tasks.}
\label{tab:auc_preservation}
\setlength{\tabcolsep}{7pt}
\renewcommand{\arraystretch}{1.0}
\begin{tabular}{lcc}
\toprule
Method & Avg. Preservation Ratio ($\uparrow$) & Std. Preservation \\
\midrule
\method (ours) & \textbf{0.977} & \textbf{0.022} \\
TEDDY & 0.894 & 0.063 \\
MoG & 0.886 & 0.077 \\
SEM & 0.884 & 0.080 \\
\bottomrule
\end{tabular}
\end{table}

Task~3 evaluates influential-node turnover prediction and is particularly important for newly developing transaction networks, as it predicts the future prospects of the underlying asset. In more established or stable networks, however, the influential-node set changes much less across snapshots, producing highly imbalanced or nearly degenerate labels. This is why Task~3 results are not reported for the TGBL datasets and for several blockchain datasets where the influential-node composition remains relatively stable throughout the observation period.

We aggregate the figure-level results in Table~\ref{tab:auc_preservation}, which further confirms that the proposed method achieves the strongest predictive preservation relative to the original full graphs. Across all valid dataset-task pairs, \method preserves $97.7\%$ of the full-graph ROC--AUC on average, outperforming TEDDY, MoG, and SEM under the same sparsification budget. These results show that the proposed Temporal Forman--Ricci curvature ranking retains substantially more predictive temporal-structural information after removing approximately $80\%$ of the edges.

\subsubsection{Curvature Selection Policy}
\label{sec:curvature_selection}

We evaluate whether different curvature-ranked edge groups preserve different levels of predictive information. To this end, we analyze the impact of the edge-removal ratio $\rho$ and the edge-selection strategy $\Pi_\rho$.

Figure~\ref{fig:edge_removal_sensitivity} shows the sensitivity of the proposed method to different edge-removal ratios on Task~1, averaged over datasets. We see a similar trend in other tasks.  The ROC--AUC remains relatively stable from 0\% to 80\% removal, decreasing only slightly from 0.789 to 0.768 even when 80\% of the edges are removed. In contrast, performance drops sharply beyond 80\% removal. Based on this trend, we use 80\% edge removal in the main experiments ($\rho=0.8$) because it provides a strong balance between graph compression and predictive preservation. 

\begin{figure}[!t]
    \centering
    \includegraphics[width=0.95\columnwidth]{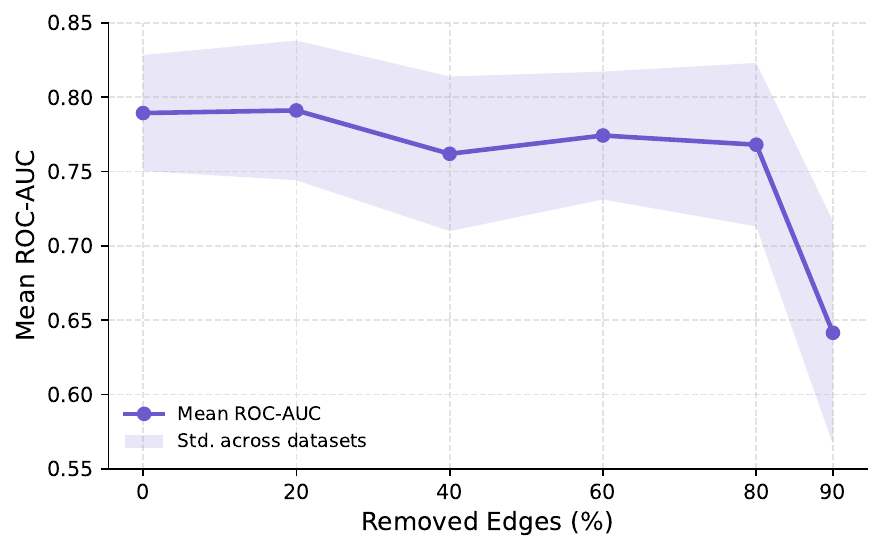}
    \caption{Sensitivity of the proposed framework to the edge-removal ratio on Task~1.}
    \label{fig:edge_removal_sensitivity}
\end{figure}

For each temporal snapshot, we partition edges into five equal-frequency bins according to their \method scores, where Bin~1 contains the lowest-curvature edges, and Bin~5 contains the highest-curvature edges.

Table~\ref{tab:our_method_bin_results} reports the ROC--AUC results across datasets and prediction tasks for each curvature-ranked bin. Although the best-performing curvature range is not identical for every dataset-task pair, the highest-curvature group provides the strongest and most consistent performance overall. In particular, Bin~5 achieves the best ROC--AUC result in the majority of evaluated settings. These results indicate that high-curvature temporal interactions contain a large fraction of the predictive information preserved in the full graph. Specifically, the observed behavior follows a Pareto-style concentration pattern: approximately 20\% of the temporal edges, selected from the highest-curvature range, retain most of the downstream performance obtained from the complete graph. This finding is consistent with prior evidence from graph sparsification and graph lottery-ticket studies, where a relatively small subset of edges was found to preserve most of the predictive signal of the full graph~\cite{spielman2008graph,chen2021unified,seo2024teddy,zhang2025graph}. In our setting, this concentration emerges from the curvature ranking itself, where high-curvature edges are less redundant in their local temporal neighborhoods and therefore carry a larger share of task-relevant information. Based on these validation results, we use $\Pi_{\rho}^{\mathrm{high}}$ as the default retention policy under the $80\%$ edge-deletion setting.

\input{tables/binauc}

\begin{table*}[tb]
\centering
\scriptsize
\caption{Runtime saving percentage across datasets. For each dataset and method, \textit{Save\%} is computed as $100(1-T_{\mathrm{sparse}}/T_{\mathrm{full}})$, where $T_{\mathrm{sparse}}=T_{\mathrm{spars.}}+T_{\mathrm{proc.}}+T_{\mathrm{pred.}}$ includes sparsification, feature processing, and prediction runtime, and $T_{\mathrm{full}}$ is the runtime of the original full-graph pipeline. The Avg. column reports the arithmetic mean across datasets.}
\label{tab:runtime_saving_pivot}
\setlength{\tabcolsep}{3.2pt}
\renewcommand{\arraystretch}{1.0}
\resizebox{\textwidth}{!}{%
\begin{tabular}{lrrrrrrrrrrrrc}
\toprule
Method & ADX & BAG & BEPRO & DERC & DINO & ETH2x & EVER & GLM & HOICHI & TGBL-CMT & TGBL-COIN & TGBL-REV & Avg. Saved (\%) \\
\midrule
\method & \textbf{47.12} & 48.60 & 50.73 & \textbf{57.22} & 51.70 & 56.91 & 49.06 & 56.28 & 48.31 & 68.74 & 59.33 & \textbf{77.04} & 55.94 \\
MoG & 33.25 & 45.63 & 47.62 & 45.50 & 47.24 & 52.10 & 51.50 & 54.20 & 40.12 & 67.40 & 41.70 & 75.78 & 50.17 \\
SEM & 2.94 & 49.23 & 47.88 & 59.14 & \textbf{58.71} & \textbf{60.56} & 52.16 & 56.77 & \textbf{56.90} & 66.82 & 42.39 & 72.58 & 52.17 \\
TEDDY & 38.92 & \textbf{56.88} & \textbf{52.33} & 52.71 & 57.96 & 58.84 & \textbf{52.34} & \textbf{61.77} & 48.90 & \textbf{69.89} & \textbf{63.42} & 76.91 & \textbf{57.24} \\
\bottomrule
\end{tabular}%
}
\end{table*}

\subsubsection{Sensitivity Analysis}
\label{sec:sensitivity}

To evaluate the robustness of the proposed framework with respect to temporal influence modeling, we conduct a sensitivity analysis on the temporal decay parameter $\tau$, which controls how rapidly the influence between temporally neighboring edges decays during curvature computation. Smaller $\tau$ values restrict the interaction range to nearly simultaneous transactions, while larger $\tau$ values allow temporally distant interactions to contribute more strongly to the curvature calculation.

\begin{figure}[!t]
    \centering
    \includegraphics[width=0.95\columnwidth]{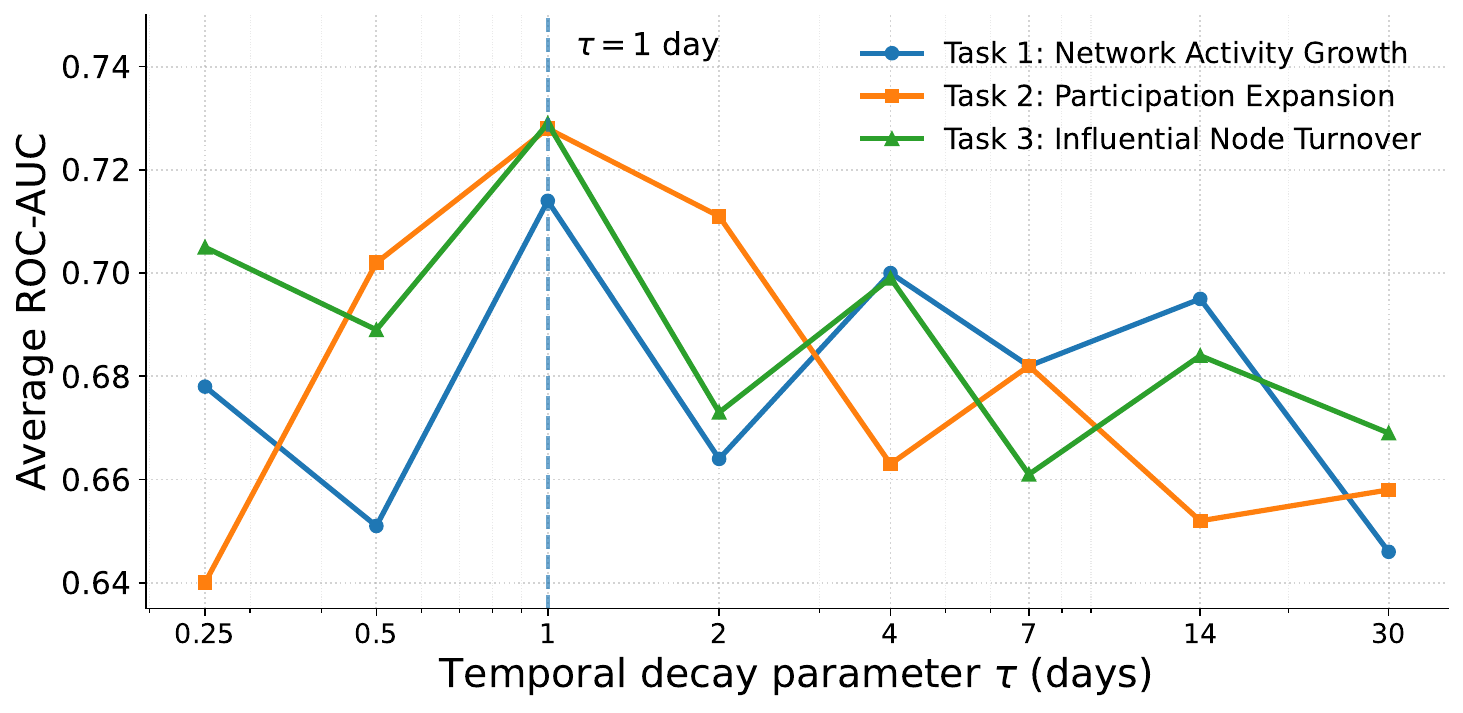}
    \caption{Sensitivity of the proposed framework to the temporal decay parameter $\tau$.}
    \label{fig:tau_sensitivity}
\end{figure}

We analyze $\tau \in \{0.25, 0.5, 1, 2, 4, 7, 14, 30\}$ in Figure~\ref{fig:tau_sensitivity}. The results indicate that \method achieves its strongest predictive performance around $\tau=1$ across most tasks. When $\tau$ becomes excessively small, the curvature computation becomes highly localized and sensitive to short-term fluctuations, leading to reduced predictive stability. In contrast, large $\tau$ values weaken the temporal discrimination capability of \method by allowing a broad range of temporally distant interactions to contribute similarly during local neighboring-edge penalty modeling. The performance degradation remains relatively moderate over a broad range of $\tau$ values, indicating that the proposed framework is reasonably robust to the choice of temporal decay scale.

\xhdr{Runtime analysis} We evaluate end-to-end runtime saving relative to the full-graph pipeline, including sparsification, feature processing, and prediction. Table~\ref{tab:runtime_saving_pivot} reports the saving percentage for each method and dataset, computed as $100(1-T_{\mathrm{sparse}}/T_{\mathrm{full}})$. \method reduces runtime on every dataset and achieves a $55.94\%$ average saving across all datasets. MoG, SEM, and TEDDY achieve average savings of $50.17\%$, $52.17\%$, and $57.24\%$, respectively. Although TEDDY has a slightly higher average runtime saving, \method provides nearly the same efficiency gain while preserving substantially stronger predictive performance, as shown in the AUC results of Table~\ref{tab:auc_preservation}.

\section{Conclusion}
\label{sec:conclusion}

We have introduced \method, a graph-curvature-based edge sparsification framework for dynamic graph learning. The method extends Forman--Ricci curvature to directed weighted temporal graphs by combining endpoint support, temporal proximity, and local outgoing interaction competition into a single edge-level score. By computing curvature within temporal snapshots, \method ranks edges according to their local temporal-structural importance and constructs sparse graph representations that preserve coverage across the full observation period.

The experiments show that high-curvature edge retention is effective across blockchain transaction networks and TGBL benchmark datasets. \method removes approximately $80\%$ of the edges while preserving $97.7\%$ of the full-graph ROC--AUC on average. It achieves the best ROC--AUC in 27 out of 30 reported dataset-task settings and in all 6 TGBL settings, outperforming SOTA competitors under the same sparsification budget. \method reduces end-to-end runtime by $55.94\%$ on average. 

These results indicate that temporal curvature is an effective criterion for selecting informative interactions in large dynamic graphs. Future work can study adaptive temporal windows, task-specific curvature thresholds, and extensions to continuous-time temporal graph learning.

\section*{GenAI Usage Disclosure}
Generative AI tools were used to assist with language editing and improving the clarity of the manuscript. No GenAI tools were used to process or analyze data, conduct experiments, or produce results.  All suggested edits were reviewed and verified by the authors.

\printbibliography

\end{document}

%% file: images/threetaskimages.tex
\begin{figure}[t]
    \centering
    
    \begin{subfigure}{0.88\columnwidth}
        \centering
        \includegraphics[width=\linewidth]{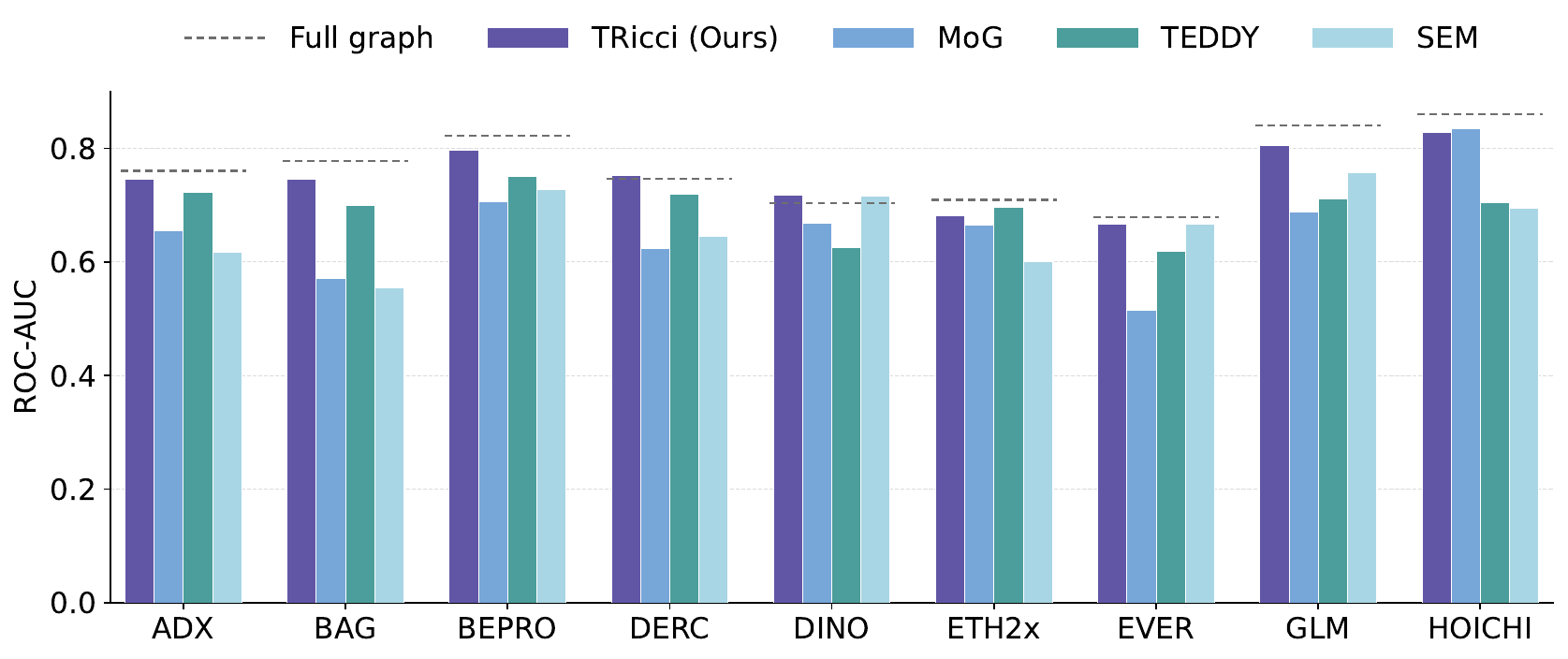}
        \caption{Task~1: network activity growth prediction}
    \end{subfigure}

    \begin{subfigure}{0.88\columnwidth}
        \centering
        \includegraphics[width=\linewidth]{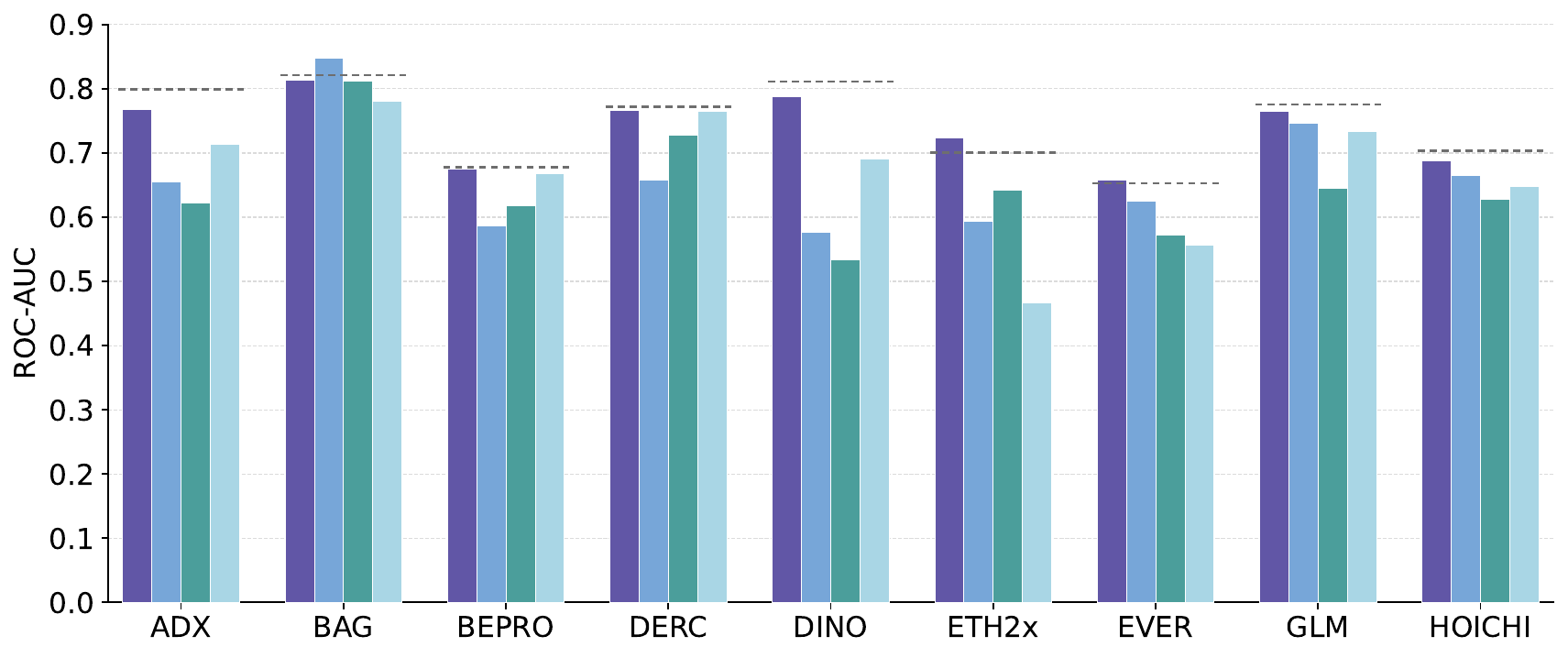}
        \caption{Task~2: network participation expansion prediction}
    \end{subfigure}

    \begin{subfigure}{0.88\columnwidth}
        \centering
        \includegraphics[width=\linewidth]{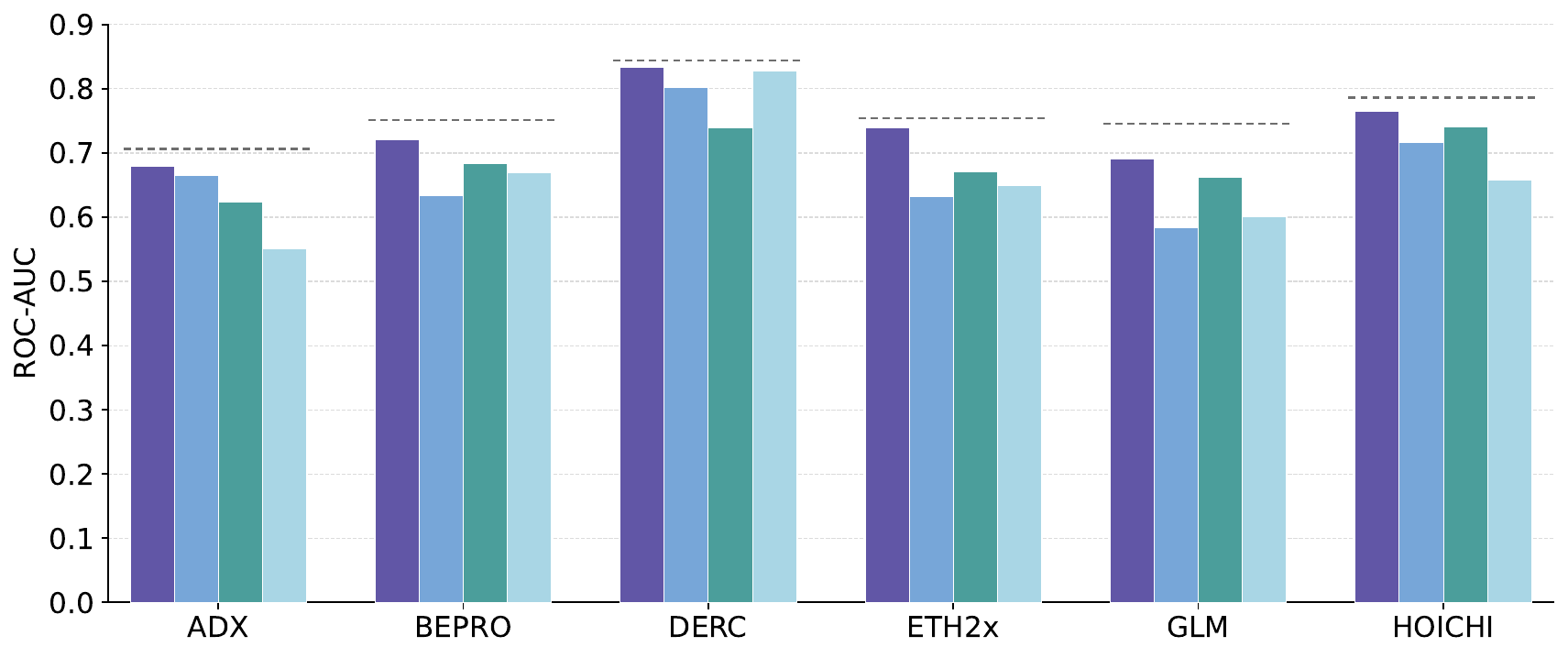}
        \caption{Task~3: influential node turnover prediction}
    \end{subfigure}

    \caption{ROC--AUC comparison on token transaction networks across the three evaluated graph-level temporal prediction tasks. Results for Task~3 are reported only for datasets with non-degenerate labels.}
    
    \label{fig:token_tasks}
\end{figure}

%% file: images/tgblimages.tex
\begin{figure}[!t]
    \centering
    
    \begin{subfigure}{0.8\columnwidth}
        \centering
        \includegraphics[width=\linewidth]{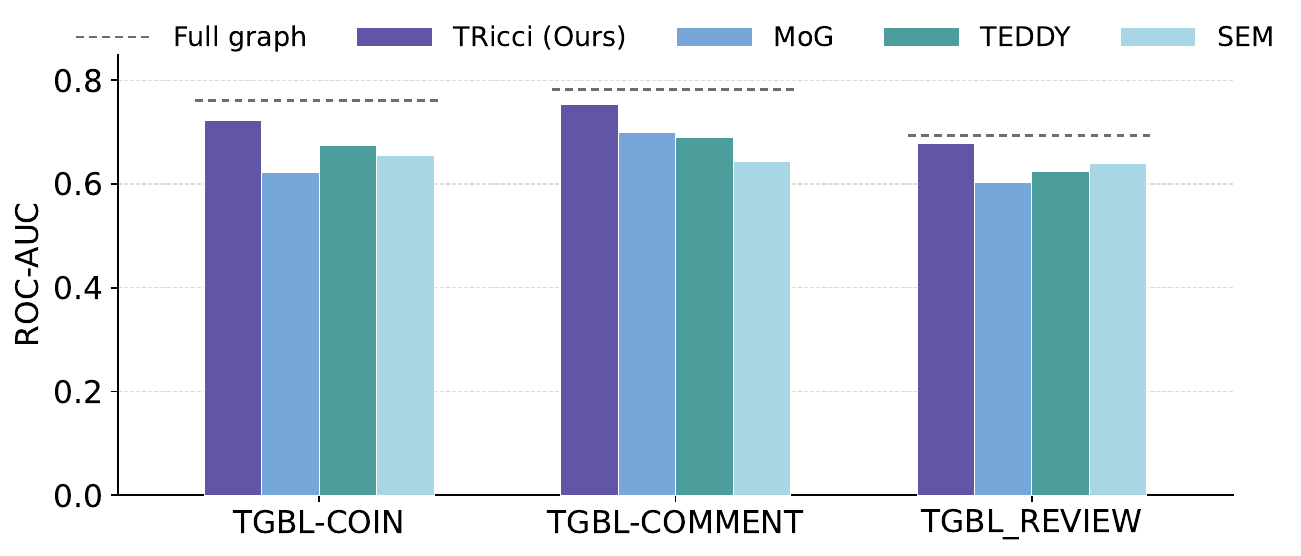}
        \caption{Task~1: network activity growth prediction}
    \end{subfigure}
    
    \begin{subfigure}{0.8\columnwidth}
        \centering
        \includegraphics[width=\linewidth]{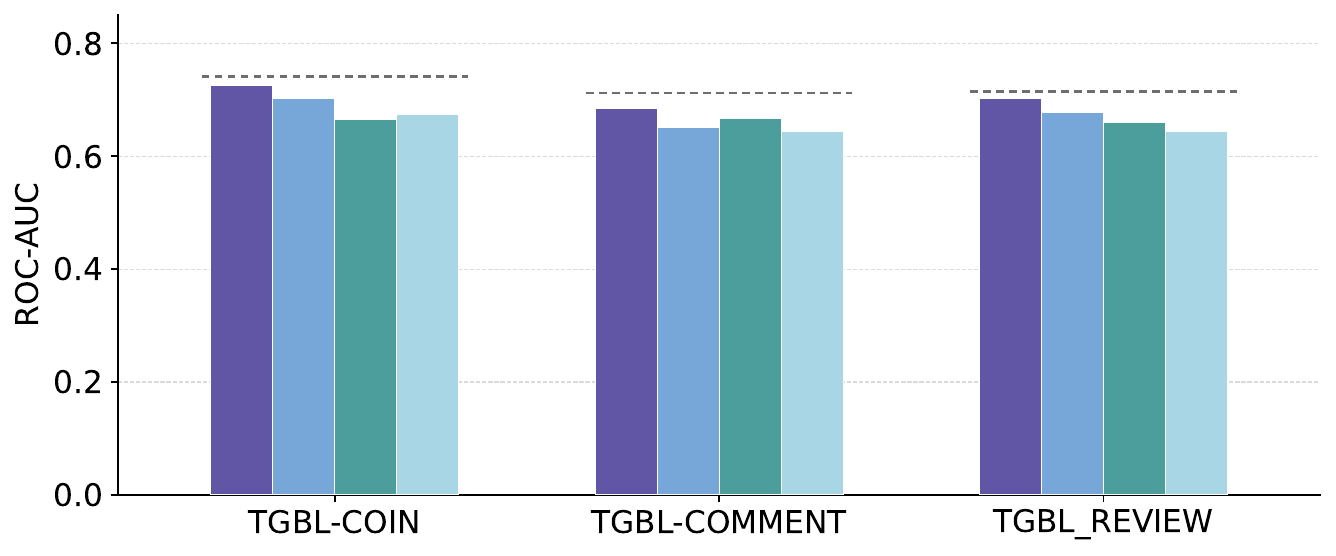}
        \caption{Task~2: network participation expansion prediction}
    \end{subfigure}
    
    \caption{ROC--AUC comparison on TGBL benchmark datasets across two graph-level temporal prediction tasks.}
    
    \label{fig:tgbl_tasks}
\end{figure}

%% file: tables/binauc.tex
\begin{table}[t]
\centering
\scriptsize
\caption{ROC--AUC results of the proposed method across curvature-ranked bins. The Full column reports the original graph performance, while Bin1--Bin5 represent increasingly higher-curvature edge groups. The best-performing bin for each dataset/task is highlighted.}
\label{tab:our_method_bin_results}
\renewcommand{\arraystretch}{1.02}
\setlength{\tabcolsep}{2.4pt}
\resizebox{0.4\textwidth}{!}{%
\begin{tabular}{lcccccc}
\toprule
\textbf{Dataset / Task} & \textbf{Full} & \textbf{Bin1} & \textbf{Bin2} & \textbf{Bin3} & \textbf{Bin4} & \textbf{Bin5} \\
\midrule
ADX / task1 & \textcolor{gray}{0.76} & 0.56 & 0.66 & 0.70 & 0.73 & \best{0.75} \\
ADX / task2 & \textcolor{gray}{0.80} & 0.57 & 0.50 & 0.65 & 0.75 & \best{0.77} \\
ADX / task3 & \textcolor{gray}{0.71} & 0.44 & 0.56 & 0.49 & 0.63 & \best{0.68} \\
\midrule
BAG / task1 & \textcolor{gray}{0.78} & 0.44 & 0.34 & 0.63 & \best{0.75} & 0.73 \\
BAG / task2 & \textcolor{gray}{0.82} & 0.58 & 0.37 & 0.76 & \best{0.81} & 0.80 \\
\midrule
BEPRO / task1 & \textcolor{gray}{0.82} & 0.67 & 0.69 & 0.72 & 0.74 & \best{0.80} \\
BEPRO / task2 & \textcolor{gray}{0.68} & 0.55 & 0.59 & 0.65 & \best{0.67} & 0.65 \\
BEPRO / task3 & \textcolor{gray}{0.75} & 0.52 & 0.56 & 0.64 & 0.68 & \best{0.72} \\
\midrule
DERC / task1 & \textcolor{gray}{0.75} & 0.56 & 0.58 & 0.60 & 0.68 & \best{0.75} \\
DERC / task2 & \textcolor{gray}{0.77} & 0.58 & 0.55 & 0.66 & \best{0.77} & 0.77 \\
DERC / task3 & \textcolor{gray}{0.84} & 0.60 & 0.59 & 0.77 & 0.69 & \best{0.83} \\
\midrule
DINO / task1 & \textcolor{gray}{0.70} & 0.44 & 0.56 & 0.50 & \best{0.72} & 0.69 \\
DINO / task2 & \textcolor{gray}{0.81} & 0.50 & 0.36 & 0.69 & 0.67 & \best{0.79} \\
\midrule
ETH2X-FLI / task1 & \textcolor{gray}{0.71} & 0.37 & 0.49 & 0.39 & 0.62 & \best{0.68} \\
ETH2X-FLI / task2 & \textcolor{gray}{0.70} & 0.51 & 0.54 & 0.39 & 0.65 & \best{0.72} \\
ETH2X-FLI / task3 & \textcolor{gray}{0.75} & 0.67 & 0.62 & 0.69 & 0.62 & \best{0.74} \\
\midrule
EVERMOON / task1 & \textcolor{gray}{0.68} & 0.36 & 0.55 & 0.26 & 0.64 & \best{0.67} \\
EVERMOON / task2 & \textcolor{gray}{0.65} & 0.35 & \best{0.66} & 0.44 & 0.51 & 0.62 \\
\midrule
GLM / task1 & \textcolor{gray}{0.84} & 0.60 & 0.65 & 0.71 & 0.77 & \best{0.80} \\
GLM / task2 & \textcolor{gray}{0.78} & 0.56 & 0.69 & 0.70 & 0.70 & \best{0.76} \\
GLM / task3 & \textcolor{gray}{0.75} & 0.67 & 0.63 & 0.54 & 0.66 & \best{0.69} \\
\midrule
HOICHI / task1 & \textcolor{gray}{0.86} & 0.65 & 0.61 & 0.62 & 0.71 & \best{0.83} \\
HOICHI / task2 & \textcolor{gray}{0.70} & 0.56 & 0.45 & 0.63 & \best{0.69} & 0.68 \\
HOICHI / task3 & \textcolor{gray}{0.79} & 0.42 & 0.67 & 0.44 & \best{0.76} & 0.76 \\
\midrule
TGBL-CMT / task1 & \textcolor{gray}{0.78} & 0.50 & 0.60 & 0.49 & 0.69 & \best{0.75} \\
TGBL-CMT / task2 & \textcolor{gray}{0.71} & 0.49 & 0.47 & 0.49 & 0.51 & \best{0.69} \\
TGBL-COIN / task1 & \textcolor{gray}{0.76} & 0.37 & 0.68 & 0.68 & 0.71 & \best{0.72} \\
TGBL-COIN / task2 & \textcolor{gray}{0.74} & 0.40 & 0.62 & 0.58 & 0.71 & \best{0.73} \\
TGBL-REV / task1 & \textcolor{gray}{0.69} & 0.49 & 0.52 & 0.48 & 0.62 & \best{0.68} \\
TGBL-REV / task2 & \textcolor{gray}{0.71} & 0.51 & 0.63 & 0.53 & \best{0.70} & 0.69 \\

\bottomrule 
\end{tabular}%
}
\end{table}